\documentclass{article}
\usepackage[preprint]{colm2026_conference}

\usepackage{amsmath}
\usepackage{amsthm}
\usepackage{float}
\usepackage{amssymb}

\newtheorem{proposition}{Proposition}

\newtheorem{definition}{Definition}

\usepackage{microtype}
\usepackage{hyperref}
\usepackage{url}
\usepackage{graphicx}
\usepackage{booktabs}
\usepackage{multirow}
\usepackage{xcolor}
\usepackage{colortbl}
\usepackage{subcaption}
\usepackage{bm}
\usepackage{enumitem}
\usepackage{tcolorbox}
\usepackage{lineno}
\usepackage{wrapfig}
\tcbuselibrary{skins,breakable}

\usepackage{caption}
\AtBeginDocument{%
  \abovedisplayskip      4pt plus 2pt minus 2pt
  \belowdisplayskip      4pt plus 2pt minus 2pt
  \abovedisplayshortskip 2pt plus 1pt
  \belowdisplayshortskip 4pt plus 2pt minus 2pt
}

\usepackage{etoolbox}
\AtBeginEnvironment{proof}{\setlength\topsep{2pt}\setlength\partopsep{0pt}}
\AtBeginEnvironment{proposition}{\setlength\topsep{2pt}\setlength\partopsep{0pt}}
\AtBeginEnvironment{corollary}{\setlength\topsep{2pt}\setlength\partopsep{0pt}}
\AtBeginEnvironment{definition}{\setlength\topsep{2pt}\setlength\partopsep{0pt}}

\definecolor{darkblue}{rgb}{0,0,0.5}
\hypersetup{colorlinks=true,citecolor=darkblue,linkcolor=darkblue,urlcolor=darkblue}

\newcommand{\clinic}[1]{\textcolor{teal!80!black}{#1}}

\newcommand{\tabdom}[1]{\textcolor{orange!85!black}{#1}}
\newcommand{\chartdom}[1]{\textcolor{violet}{#1}}

\usepackage[textsize=scriptsize]{todonotes}
\usepackage{inconsolata}

\title{When Verification Fails: How Compositionally Infeasible Claims Escape Rejection}

\author{
Muxin Liu$^{*}$ \quad Delip Rao \quad Grace Kim \quad Chris Callison-Burch \\
University of Pennsylvania \\
\texttt{muxinliu@seas.upenn.edu}
}

\begin{document}
\ifcolmsubmission
\linenumbers
\fi
\maketitle

\begin{abstract}
Scientific claim verification, the task of determining whether claims are entailed by scientific evidence, is fundamental to establishing discoveries in evidence while preventing misinformation. This process involves evaluating each asserted constraint against validated evidence. Under the Closed-World Assumption (CWA), a claim is accepted if and only if all asserted constraints are positively supported. We show that existing verification benchmarks cannot distinguish models enforcing this standard from models applying a simpler shortcut called salient-constraint checking, which applies CWA's rejection criterion only to the most salient constraint and accepts when that constraint is supported. Because existing benchmarks construct infeasible claims by perturbing a single salient element they are insufficient at distinguishing between rigorous claim verification and simple salient-constraint reliance.
To separate the two, we construct compositionally infeasible claims where the salient constraint is supported but a non-salient constraint is contradicted. Across model families and modalities, models 
that otherwise saturate existing benchmarks consistently over-accept these claims, confirming the prevalence of such shortcut reasoning. Via model context interventions, we show that different models and prompting strategies occupy distinct positions on a shared ROC curve, indicating that the gap between model families reflects differences in verification threshold rather than underlying reasoning ability, and that the compositional inference bottleneck is a structural property of current verification behavior that strategy guidance alone cannot overcome.
\end{abstract}

\section{Introduction}

Claim verification, the task of determining whether claims are entailed by evidence, is fundamental to grounding discoveries in evidence while preventing misinformation. Claim verification has been widely studied as a natural language inference (NLI) task, with large-scale benchmarks spanning general knowledge verification~\citep{thorne2018fever,aly2021fact}, scientific literature~\citep{wadden2020fact,wadden2022scifact}, clinical trials~\citep{jullien2023nli4ct,vladika2024healthfc}, tables~\citep{chen2019tabfact, lu2023scitab}, and charts~\citep{akhtar2023reading, akhtar2024chartcheck,masry2025chartqapro,wang2025sciver}. This judgment requires a precise epistemic standard: a claim should be accepted if and only if the evidence provides positive support for it, and absence of support warrants rejection. This standard is formalized as the Closed-World Assumption (CWA) ~\citep{clark1977negation, reiter1981closed}. Operationally, this requires checking every constraint the claim asserts~\citep{min2023factscore}, including those derived compositionally from multiple evidence units ~\citep{si2021topic,zhang2024causal, chen2024complex}.

We evaluate Large Language Models (LLMs) across three established benchmarks, NLI4CT~\citep{jullien2023nli4ct,jullien2024semeval} for clinical trial text, SCITAB~\citep{lu2023scitab} for scientific tables, and SciVer~\citep{wang2025sciver} for charts, and we find that models consistently achieve high accuracy on infeasible claims but lower accuracy on feasible ones, as shown in Figure~\ref{fig:ablation1}. This asymmetry could be interpreted as evidence of conservative, closed-world verification, or could reflect a weaker shortcut observed in the literature for other tasks as well ~\citep{gururangan2018annotation, mccoy2019right, liu2023evaluating}. 

The construction of existing benchmarks makes full CWA and a simpler shortcut procedure behaviorally indistinguishable. Negative examples in these benchmarks (also called standard negatives) are created by perturbing a single salient element, such as flipping a number, reversing a trend, or swapping an entity label. This makes the perturbed constraint both the most checkable part of the claim and sufficient for rejection. As a result, as shown in Figure~\ref{fig:overview}, a shortcut is sufficient when verifying \textbf{Existing Benchmark Infeasible Claim}: identify the most salient constraint in the claim, check whether it is supported by the evidence, and decide on that basis alone. This salient-constraint checking procedure is an easier, less formal approximation of CWA, as it applies CWA's rejection criterion but only to the most prominent constraint, skipping the exhaustive check that full CWA requires. It also explains the observed asymmetry. For infeasible claims, the salient constraint is violated, so rejection follows directly. For feasible claims, acceptance requires the salient constraint to be clearly retrievable, and any uncertainty defaults to rejection instead. High negative accuracy together with lower positive accuracy is therefore consistent with both full CWA and salient-constraint checking.

\begin{figure}[t]
  \centering
  \includegraphics[width=\linewidth]{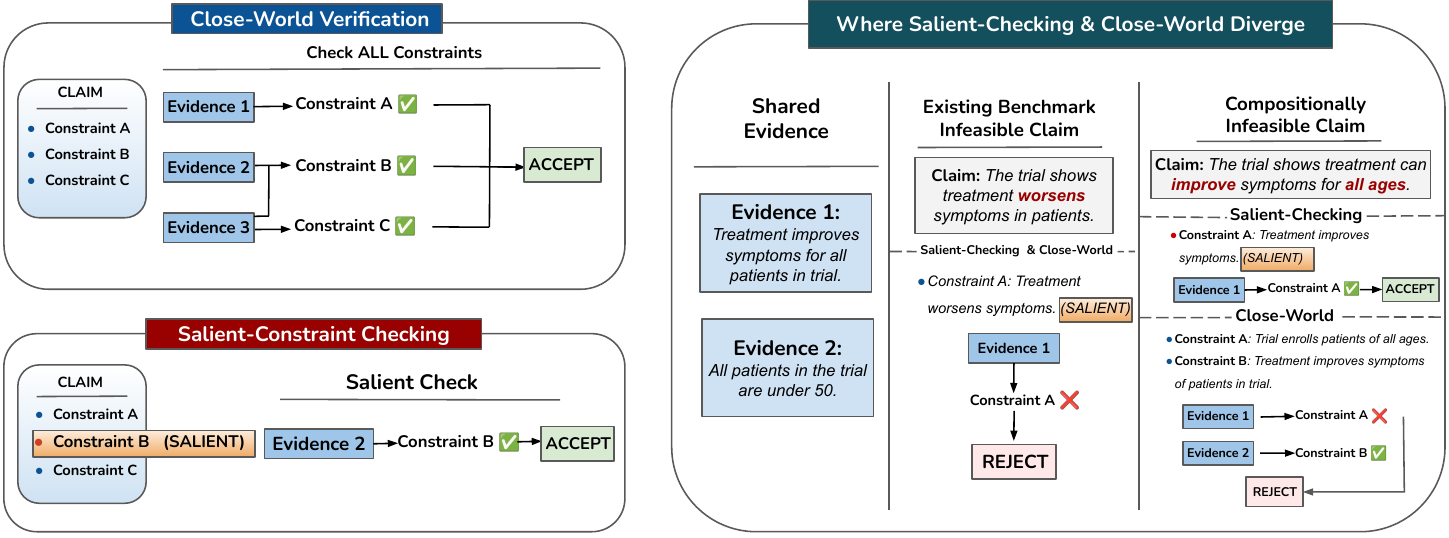}
  \caption{CWA verification checks all constraints; salient-constraint 
  checking only finds the most prominent constraint. On existing benchmarks both agree since violated constraints are always salient. On compositionally infeasible claims, salient-constraint checking accepts while CWA rejects.}
  \label{fig:overview}
\end{figure}
  
To expose the gap between salient-constraint checking and full CWA, we construct claims for which the shortcut yields acceptance but CWA requires rejection, as shown in Figure~\ref{fig:overview}, under the \textbf{Compositionally Infeasible Claim} setting. These claims, which we call compositionally infeasible claims (also called adversarial negatives), are designed so that the salient constraint appears supported by the evidence, while the violation lies in a non-salient constraint such as a composed relationship across multiple parts of the evidence. We create hard negatives using a graph-based corrupt-and-propagate procedure, which preserves the salient observations while introducing violations through composing multiple parts of the same evidence source such as two sentences in a clinical report, two trends in a chart, or two columns in a table. Salient-constraint checking verifies a single constraint and accepts, while full CWA checks all constraints including non-salient ones and rejects.

\begin{wrapfigure}{r}{0.5\columnwidth}
  \centering
  \includegraphics[width=0.5\columnwidth]{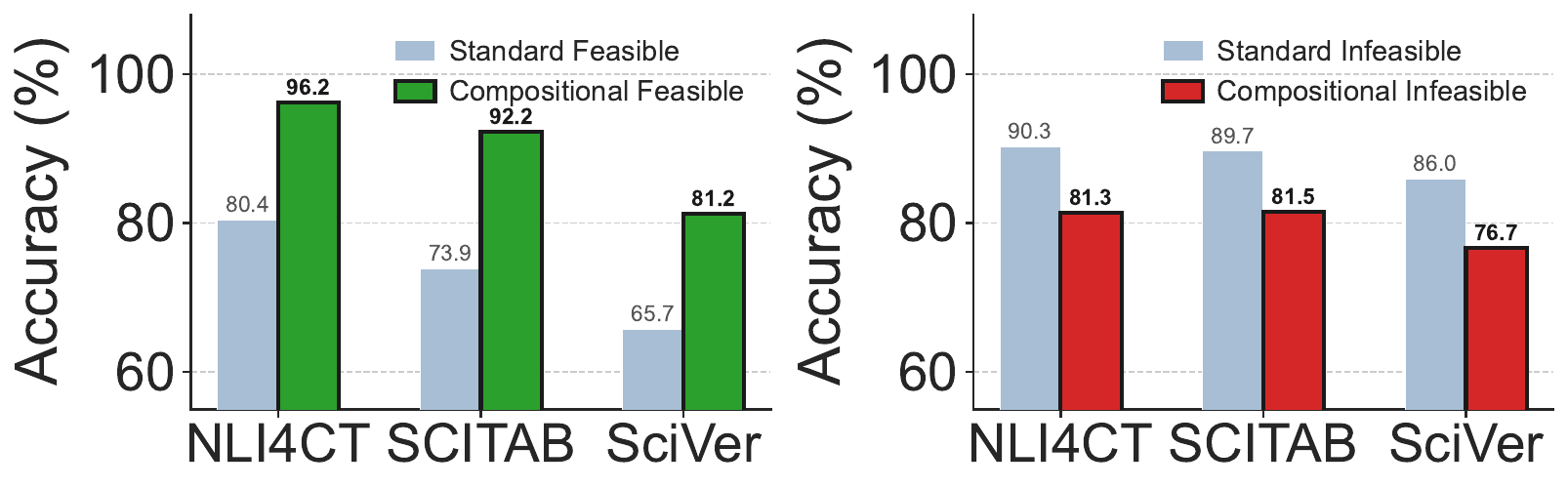}
  \caption{Averaged accuracy across models. 
  Compositionally feasible claims score higher ($\uparrow$) than standard 
  positives; compositionally infeasible claims score lower ($\downarrow$). 
  Compositional structure shifts behavior toward acceptance.}
  \label{fig:ablation1}
    \vspace{-15pt}  
\end{wrapfigure}

Models that perform strongly on existing benchmarks fail consistently 
on adversarial examples, as summarized in Figure~\ref{fig:ablation1}: they accept adversarial negatives at 
higher rates than standard negatives, and this gap persists across 
model families and scales. Models accept claims when the salient 
constraint appears supported, which produces systematic over-acceptance 
on adversarial negatives.

We further study this failure with test-time guidance toward CWA, including closed-world instructions and explicit decomposition. These interventions shift the tradeoff between acceptance and rejection rather than improving performance on adversarial negatives overall. This movement along the ROC curve suggests that different guidance strategies change how the model applies verification, but do not resolve the underlying reasoning bottleneck. To analyze reasoning more directly, we also introduce a rubric-based LLM judge that scores how well a model's reasoning trace follows CWA, including whether it checks all constraints the claim asserts or terminates early at the salient one. These scores reveal systematic differences across models, examples, and interventions, and provide a trace-level analysis of verification reasoning beyond outcome-based evaluation.

We make three contributions. First, we identify a structural limitation 
of existing verification benchmarks: single-constraint construction makes 
salient-constraint checking and full CWA behaviorally indistinguishable. Second, we construct compositionally infeasible claims across three 
modalities that expose this gap, and show consistent over-acceptance across 
model families and scales. Third, we show this failure reflects a reasoning 
bottleneck rather than a strategy choice through both intervention and trace analysis.

\section{Theoretical Framework}
We formalize the verification standards for scientific claim 
verification, introduce salient-constraint checking as a weaker procedure 
that existing benchmarks cannot distinguish from full CWA, and define 
compositional falsification as the condition that separates the two
\subsection{Closed-World and Open-World Assumptions}

Let $E = \{e_1, \ldots, e_m\}$ be a set of evidence pieces and $c$ a claim
asserting constraints $\mathcal{C}(c) = \{c_1, \ldots, c_n\}$. Under
\textbf{CWA}, $c$ is accepted iff every
constraint is supported:
\begin{equation}
\text{CWA}(c, E) = \begin{cases}
\text{Accept} & \text{if } \forall c_i \in \mathcal{C}(c),\
               \exists S_i \subseteq E\ \text{s.t.}\
               \bigwedge_{e \in S_i} e \models c_i \\
\text{Reject} & \text{otherwise}
\end{cases}
\end{equation}
CWA is inherently \textbf{$\forall$-type}: every constraint must be checked.
Under the \textbf{Open-World Assumption (OWA)}, a claim is rejected iff explicitly contradicted:
\begin{equation}
\text{OWA}(c, E) = \begin{cases}
\text{Reject} & \text{if } \exists S \subseteq E\ \text{s.t.}\
               \bigwedge_{e \in S} e \models \neg c \\
\text{Accept} & \text{otherwise}
\end{cases}
\end{equation}
OWA is inherently \textbf{$\exists$-type}: it searches for any contradiction
and accepts by default when none is found.

\subsection{Salient-Constraint Checking}
\label{sec:scc}

\begin{definition}[Salience]
The $\text{salience}(c_i, E)$ of a constraint 
$c_i \in \mathcal{C}(c)$ measures how directly $E$ speaks 
to $c_i$ without requiring composition across multiple 
evidence units. The most salient constraint is 
$c^* = \arg\max_{c_i \in \mathcal{C}(c)} \text{salience}(c_i, E)$, 
corresponding operationally to the primary subject-predicate 
assertion verifiable from a single contiguous evidence unit.
\end{definition}

\textbf{Salient-constraint checking} applies the CWA verdict to $c^*$ alone:
\begin{equation}
\text{V}_s(c, E) = \begin{cases}
\text{Accept} & \text{if } \exists e \in E\ \text{s.t.}\ e \models c^* \\
\text{Reject} & \text{otherwise}
\end{cases}
\end{equation}
$\text{V}_s$ is a hybrid: it inherits CWA's rejection criterion but 
reduces the $\forall$ over all constraints to a $\exists$ over one 
salient constraint, making it operationally similar to OWA. Let $c^+, c^-$ 
denote feasible/infeasible claim. Existing benchmarks construct $c^-$ by 
perturbing a single salient element, ensuring $\exists!\, c^* \in 
\mathcal{C}(c^-)$ s.t.\ $e \models \neg c^*$ for some $e \in E$, with 
$\forall c_i \neq c^*$: $\exists e \in E$ s.t.\ $e \models c_i$. We call it \textit{single-constraint regime}.

\begin{proposition}
In the single-constraint regime, $\text{CWA}$ and $\text{V}_s$ produce
identical verdicts on both $c^-$ and $c^+$, assuming all constraints of
$c^+$ are supported.
\end{proposition}
\begin{proof}
For $c^-$: both check $c^*$, find it contradicted, and reject. For $c^+$:
CWA finds all constraints supported; $\text{V}_s$ finds $c^*$ supported.
Both accept.
\end{proof}
Thus, existing benchmark performance cannot distinguish $\text{CWA}$ from $\text{V}_s$.

\subsection{Compositional Falsification}

A \textbf{compositionally infeasible claim} $c^\dagger$ satisfies:
\begin{equation}
\exists e \in E\ \text{s.t.}\ e \models c^*; \qquad
\exists\, c_j \neq c^*,\ \exists S_2 \subseteq E\ \text{s.t.}\
\bigwedge_{e \in S_2} e \models \neg c_j
\end{equation}
The violation may be direct ($|S_2|=1$) or compositional ($|S_2|\geq 2$).

\begin{proposition}
On compositionally infeasible claims, $\text{V}_s$ accepts and
$\text{CWA}$ rejects.
\end{proposition}
\begin{proof}
$\text{V}_s$ finds $c^*$ supported and accepts; CWA finds $c_j$ contradicted and rejects.
\end{proof}

\begin{proposition}
Degradation on compositionally infeasible claims relative to
single-constraint infeasible claims is diagnostic for $\text{V}_s$ as the
default. Absence of degradation supports $\text{CWA}$.
\end{proposition}

Proposition~3 follows from Propositions~1--2. We test this in
\S\ref{sec:exp_main} and find consistent degradation across model families
and scales, confirming $\text{V}_s$ as the default verification procedure.
\section{Data Generation Method}
A limitation of existing benchmarks is their reliance on surface-level perturbations (e.g., numerical flipping, trend reversal) to construct infeasible claims. Such single-constraint violation procedures can be detectable from merely a single sentence or a table cell. These correspond to the single-constraint regime where salient-constraint checking and CWA are indistinguishable (\S\ref{sec:scc}). Our method instead constructs compositionally infeasible claims where the salient constraint remains supported and the violation lies in a non-salient constraint, detectable only by combining multiple parts of the same evidence source.
\begin{figure}[H]
  \centering
  \includegraphics[width=\columnwidth]{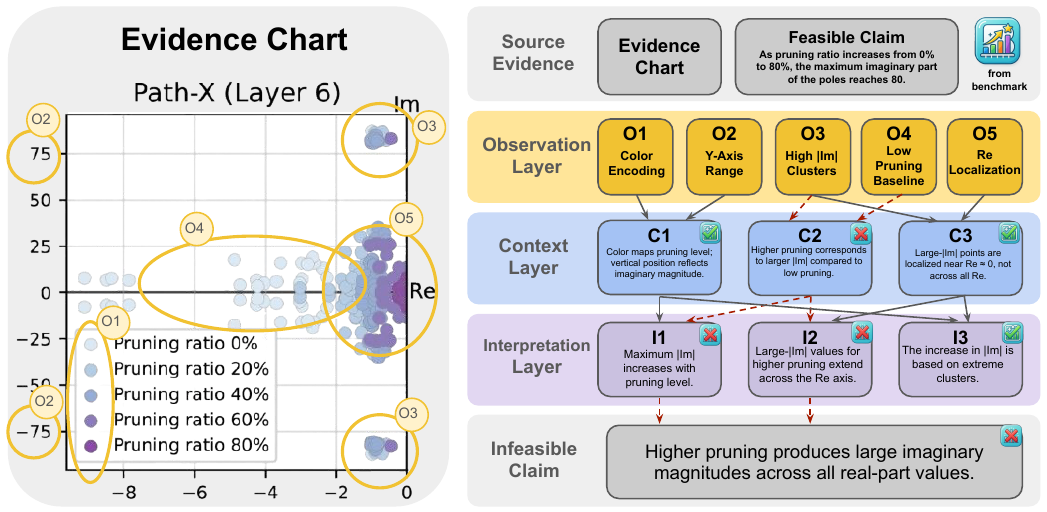}
  \caption{Graph-based adversarial construction. Observation nodes ground 
directly retrievable facts; context and interpretation nodes combine them 
into higher-level conclusions. Corrupting one non-salient node produces 
a compositionally infeasible claim while preserving all observation-layer 
support.}
  \label{fig:iceberg}
\end{figure}
\subsection{Graph-Based Adversarial Construction}
Starting from a feasible claim and its evidence source from an existing human-verified benchmark, we build a three-layer reasoning graph (Figure~\ref{fig:iceberg}). The observation layer consists of five nodes, each grounding a directly retrievable fact from the evidence without requiring inference. The context layer consists of three nodes, each combining at least two observation nodes into an intermediate relationship. The interpretation layer consists of three nodes, each combining at least two context nodes into a higher-level conclusion.

To produce a compositionally infeasible claim, we corrupt a single context or interpretation node using operations such as scope swapping or overgeneralization. The corruption is propagated through the graph to produce an infeasible claim that preserves all observation-layer support, keeping the salient constraint intact, while introducing a violation in a non-salient constraint. By preserving all observation-layer nodes, the construction ensures retrieval difficulty is matched to standard infeasible claims.

\subsection{Datasets}
\label{sec:datasets}


We construct adversarial datasets from three existing scientific claim verification benchmarks: NLI4CT, SCITAB, and SciVer. For each dataset we sample 200 standard positive and negative claims from the original benchmark, then generate a pool of adversarial negatives using our graph-based construction method implemented with GPT-5-mini. After human validation to remove ambiguous and accidentally feasible claims, we retain approximately 200 adversarial negatives per domain. Full dataset construction details are in Appendix~\ref{app:datasets}, human annotation details in Appendix~\ref{app:annotate}, prompts in Appendix ~\ref{app:gen prompt} and generation examples in Appendix ~\ref{app:gen example}.

\section{Experiments}
\label{sec:eval}

\begin{table}[t]
  \centering
  \resizebox{\linewidth}{!}{%
  \begin{tabular}{l | ccc c !{\color{gray}\vrule} r | ccc c !{\color{gray}\vrule} r | ccc c !{\color{gray}\vrule} r}
    \toprule
    & \multicolumn{5}{c|}{\textbf{NLI4CT}}
    & \multicolumn{5}{c|}{\textbf{SCITAB}}
    & \multicolumn{5}{c}{\textbf{SciVer}} \\
    \cmidrule(lr){2-6}\cmidrule(lr){7-11}\cmidrule(lr){12-16}
    \multirow{2}{*}{\textbf{Model}}
      & \multicolumn{3}{c}{\textbf{Standard}} & \textbf{Adv} & \multirow{2}{*}{$\bm{\Delta}$\textbf{(\%)}}
      & \multicolumn{3}{c}{\textbf{Standard}} & \textbf{Adv} & \multirow{2}{*}{$\bm{\Delta}$\textbf{(\%)}}
      & \multicolumn{3}{c}{\textbf{Standard}} & \textbf{Adv} & \multirow{2}{*}{$\bm{\Delta}$\textbf{(\%)}} \\
    \cmidrule(lr){2-4}\cmidrule(lr){7-9}\cmidrule(lr){12-14}
      & \textbf{Pos.(\%)} & \textbf{Neg.(\%)} & $\bm{F_1}$
      & \textbf{Neg.(\%)} &
      & \textbf{Pos.(\%)} & \textbf{Neg.(\%)} & $\bm{F_1}$
      & \textbf{Neg.(\%)} &
      & \textbf{Pos.(\%)} & \textbf{Neg.(\%)} & $\bm{F_1}$
      & \textbf{Neg.(\%)} & \\
    \midrule
    GPT-5.4          & 80.2 & 94.1 & 0.86 & 88.7 & $\downarrow$5.4
                     & 75.5 & 93.0 & 0.83 & 91.0 & $\downarrow$2.0
                     & 55.9 & 92.5 & 0.68 & 87.7 & $\downarrow$4.8 \\
    GPT-5-mini        & 83.7 & 91.1 & 0.87 & 84.7 & $\downarrow$6.4
                     & 79.0 & 92.0 & 0.84 & 82.5 & $\downarrow$9.5
                     & 78.4 & 84.5 & 0.81 & 73.4 & \textbf{$\downarrow$11.2} \\
    GPT-o4-mini       & 86.1 & 91.1 & 0.88 & 76.8 & \textbf{$\downarrow$14.2}
                     & 84.0 & 88.0 & 0.86 & 77.0 & \textbf{$\downarrow$11.0}
                     & 78.9 & 87.8 & 0.83 & 63.4 & \textbf{$\downarrow$24.4} \\
    GPT-4o-mini      & 71.3 & 85.1 & 0.77 & 73.4 & \textbf{$\downarrow$11.7}
                     & 58.5 & 81.0 & 0.66 & 67.0 & \textbf{$\downarrow$14.0}
                     & 66.7 & 59.2 & 0.64 & 50.7 & $\downarrow$8.4 \\
    \midrule
    Gemini-3.0-Flash & 86.1 & 86.6 & 0.86 & 73.9 & \textbf{$\downarrow$12.7}
                     & 87.0 & 89.0 & 0.88 & 74.0 & \textbf{$\downarrow$15.0}
                     & 75.6 & 87.3 & 0.80 & 60.6 & \textbf{$\downarrow$26.7} \\
    Gemini-2.5-Flash & 84.2 & 91.6 & 0.87 & 82.3 & $\downarrow$9.3
                     & 66.5 & 91.5 & 0.76 & 83.5 & $\downarrow$8.0
                     & 65.3 & 90.6 & 0.75 & 76.1 & \textbf{$\downarrow$14.5} \\
    Claude-Sonnet-4.5 & 75.7 & 91.6 & 0.82 & 86.2 & $\downarrow$5.4
                     & 75.0 & 92.5 & 0.82 & 89.5 & $\downarrow$3.0
                     & 52.6 & 94.8 & 0.67 & 83.1 & \textbf{$\downarrow$11.7} \\
    Claude-Haiku-4.5 & 76.2 & 91.1 & 0.82 & 84.7 & $\downarrow$6.4
                     & 65.5 & 90.5 & 0.75 & 87.5 & $\downarrow$3.0
                     & 52.6 & 91.1 & 0.65 & 85.9 & $\downarrow$5.1 \\
    \bottomrule
  \end{tabular}}
  \caption{Evaluation across three domains.
    Pos./Neg.\ = accuracy (\%) on standard positive/negative claims;
    $F_1$ = macro-$F_1$ on standard negatives;
    Adv Neg.\ = accuracy (\%) on adversarial negatives;
    $\Delta$ = Adv Neg.\ $-$ Standard Neg.\ (\%).
    $\Delta$ values shown as $\downarrow$; \textbf{bold} entries indicate substantial drop ($|\Delta|>10\%$).}
  \label{tab:routing}
\end{table}

We present three experiments that together characterize salient-constraint 
checking as the default verification procedure. \S\ref{sec:exp_main} 
establishes the performance gap under compositional falsification across 
model families and scales. \S\ref{sec:exp_graph} isolates the failure at 
the reasoning layer. \S\ref{sec:exp_prompting} tests whether the failure 
reflects a reasoning bottleneck or a strategy choice.

\subsection{Performance Gap under Compositional Falsification}
\label{sec:exp_main}

We evaluate eight proprietary models (GPT-5.4, GPT-5-mini, GPT-o4-mini, 
GPT-4o-mini, Gemini-3.0-Flash, Gemini-2.5-Flash, Claude-Sonnet-4.5, 
Claude-Haiku-4.5) and the Qwen3 family on NLI4CT, SCITAB, and SciVer, 
using standard and compositionally infeasible claims from 
\S\ref{sec:datasets}. Proprietary models are evaluated at temperature~1; Qwen3 uses thinking mode at temperature~0.6. 
All evaluations use 200 examples per domain; observed gaps exceed sampling uncertainty ($\pm$3\% at $n{=}200$).

Table~\ref{tab:routing} shows a consistent performance gap across all
models and domains. Every model shows lower accuracy on compositionally
infeasible claims than on standard infeasible claims on NLI4CT ($\Delta$
from $-3.0$ to $-12.7$), with larger drops on SCITAB and SciVer:
Gemini-3.0-Flash falls 15.0 and 26.7 points respectively. This pattern
is not explained by original-benchmark $F_1$: Gemini-3.0-Flash achieves
$F_1 = 0.88$ yet shows among the largest drops, while Claude-Sonnet-4.5
has lower $F_1 = 0.82$ with consistently smaller drops. The Claude
family's robustness is better attributed to safety-induced over-refusal
than stronger compositional verification~\citep{rottger2024xstest, huang2025safety}: their positive accuracy is among
the lowest in the table, consistent with broad rejection across claim types
rather than selective compositional checking.

\begin{figure}[t]
  \centering
  \includegraphics[width=\linewidth]{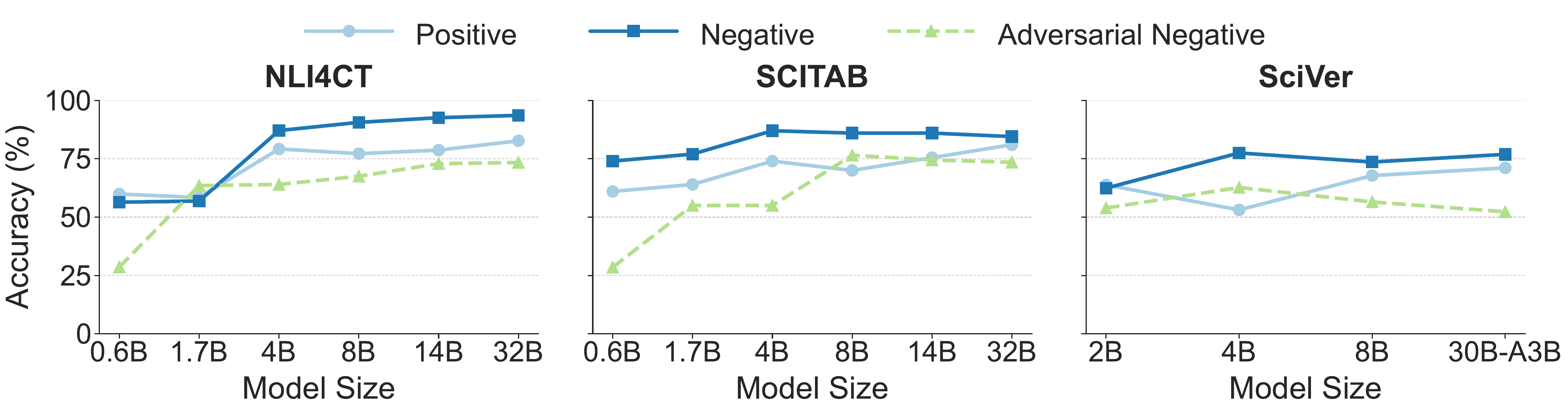}
  \caption{Scaling behavior across Qwen3 (VL) model sizes for three domains. Curves show accuracy for positive, negative, and adversarial negative claims.}
  \label{fig:scaling}
\end{figure}

Figure~\ref{fig:scaling} shows the gap persists with scale. At small scales, standard negative accuracy rises steeply while adversarial negative accuracy lags far behind, consistent with CWA-style rejection without reasoning. As scale increases, adversarial negative accuracy improves, but the gap never closes: at 32B, NLI4CT shows a 20-point gap (93.6\% vs.\ 73.4\%), SCITAB a narrower but persistent gap, and SciVer limited gains throughout. The trajectory confirms that scaling sharpens retrieval and direct-falsification rejection but does not resolve failures that require compositional inference. These results support Proposition~3 across model families, scales, and modalities. The consistent pattern of high accuracy on directly falsifiable negatives and substantially lower on adversarial negatives is not an artifact of a particular architecture or scale, but a structural feature of current verification behavior.

\subsection{Separating Retrieval from Compositional Reasoning}
\label{sec:exp_graph}

\begin{figure}[H]
  \centering
  \includegraphics[width=\linewidth]{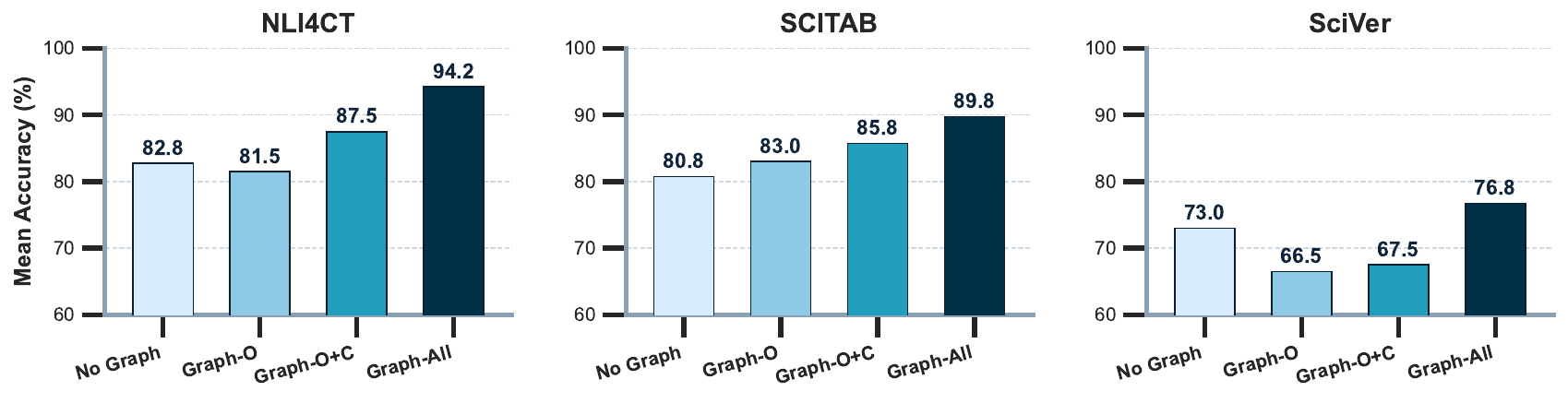}
    \caption{Adversarial negative accuracy under four evidence conditions 
    across eight models. Graph-O provides no consistent improvement over 
    no graph, while Graph-O+C produces intermediate gains and Graph-All 
    approaches ceiling on NLI4CT and SCITAB.}
  \label{fig:reasoning_graph}
\end{figure}

Models may apply salient-constraint checking for two distinct reasons. First, they may fail to retrieve the non-salient observations that reveal the violation, leaving no basis for checking beyond the salient constraint. Second, the relevant observations may be retrievable, but models cannot compose them into the joint inference required to detect the violation~\citep{barnett2024seven}. We design an experiment to distinguish between these possibilities.

We evaluate four evidence conditions. In \textbf{No graph}, the model receives only the claim and the raw document corpus. In \textbf{Graph-O}, the model additionally receives isolated atomic observations as the relevant evidence pieces without any compositional structure. In \textbf{Graph-O+C}, the model further receives context nodes explaining how observations compose into intermediate relationships. In \textbf{Graph-All}, the model receives the full graph including interpretation nodes that encode the conclusion reached by composing the context, and the model must identify where the infeasible claim diverges from this provided conclusion.

Figure~\ref{fig:reasoning_graph} shows an incremental pattern. Graph-O shows no consistent improvement over no-graph: performance plateaus on 
NLI4CT, increases slightly on SCITAB, and drops on SciVer, suggesting retrieval is not the primary bottleneck. Graph-O+C produces intermediate 
gains across all domains, as context nodes provide the intermediate compositional relationships between observations without supplying the 
final conclusion. Graph-All approaches ceiling performance on NLI4CT and SCITAB, because interpretation nodes encode the result of 
full composition directly, reducing the task to identifying where the infeasible claim diverges from the provided conclusion; on SciVer, Graph-All recovers to baseline rather than exceeding it, likely because natural language descriptions interfere with direct visual perception ~\citep{sim2025can}. The bottleneck lies specifically between observation and interpretation: models have the relevant facts and can follow compositional structure when provided, but cannot perform the compositional inference step themselves ~\citep{sim2025can}.

\subsection{Shifting Verification along the OWA--CWA Axis}
\label{sec:routing}\label{sec:exp_prompting}

\begin{figure}[H]
  \centering
  \begin{subfigure}[t]{0.76\linewidth}
    \includegraphics[width=\linewidth]{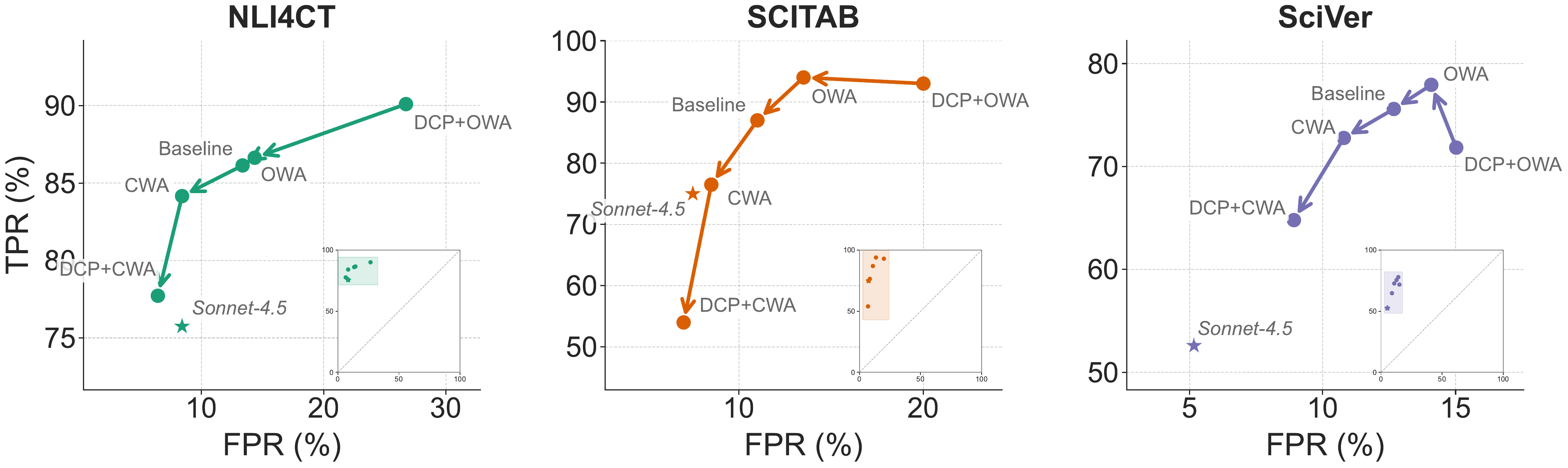}
    \phantomcaption\label{fig:trajectory}
  \end{subfigure}
  \hfill
  \begin{subfigure}[t]{0.22\linewidth}
    \includegraphics[width=\linewidth]{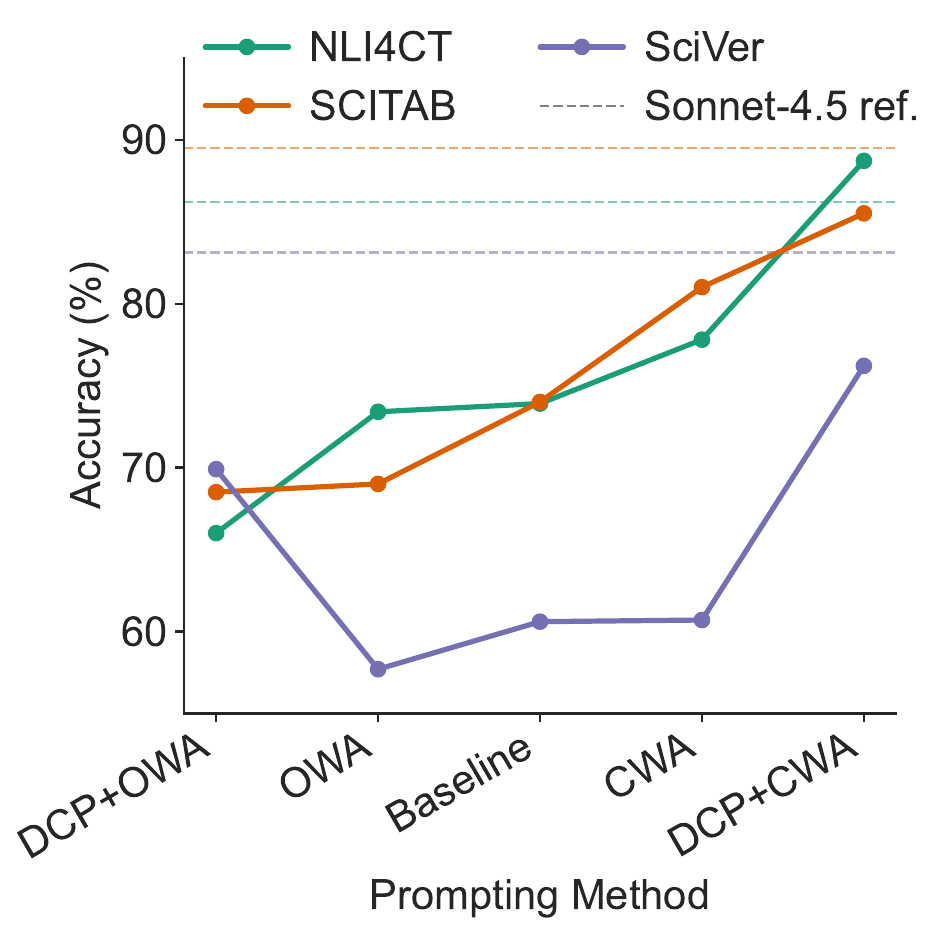}
    \phantomcaption\label{fig:hard_neg_acc}
  \end{subfigure}
  \caption{\textbf{(a)} Verification trajectories in ROC space (TPR vs.\ FPR) under different prompting conditions for Gemini-3.0-Flash. Arrows indicate the ordering of prompting variants (DCP+OWA $\to$ OWA $\to$ Baseline $\to$ CWA $\to$ DCP+CWA); stars denote the Sonnet-4.5 reference. \textbf{(b)} Adversarial negative accuracy under each prompting condition; dashed lines indicate the Sonnet-4.5 reference.}
  \label{fig:capacity}
\end{figure}

A further question is whether this reflects a reasoning ability limit or a failure to apply the right verification strategy. We intervene on Gemini-3.0-Flash with four prompting conditions with different positions on the OWA--CWA axis: \textbf{OWA}, \textbf{CWA}, and decomposition (DCP) variants \textbf{DCP+OWA} and \textbf{DCP+CWA} that explicitly enumerate subclaims before applying the verification rule. Decomposition is included because models prompted with CWA may not actively enumerate constraints without explicit instruction ~\citep{zhou2022least, dhuliawala2024chain}. DCP makes this an explicit step, ensuring the $\forall$-type checking CWA demands is actually carried out .

Figure~\ref{fig:trajectory} shows these conditions produce a smooth trajectory in ROC space ordered from most permissive to most strict (DCP+OWA $\to$ OWA $\to$ Baseline $\to$ CWA $\to$ DCP+CWA), with the baseline lying between OWA and CWA, which empirically align with \S\ref{sec:scc} that salient-constraint checking is a hybrid of the two. Claude-Sonnet-4.5 lies on the same ROC curve as Gemini-3.0-Flash under intervention, and CWA or DCP+CWA prompting moves Gemini-3.0-Flash toward its operating point. This indicates that the performance gap between model families reflects differences in default verification threshold and strategy rather than differences in underlying reasoning ability. Both models share the same reasoning capability frontier, but apply different operating points by default. Claude-Sonnet-4.5's robustness to hard negatives, attributed in \S\ref{sec:eval} to conservative rejection rather than compositional reasoning, is consistent with this picture, as safety alignment shifts its default operating point toward greater conservatism without changing the underlying capability frontier.



The trajectory also reveals the limit of strategy alone. Even under DCP+CWA, the strictest condition, hard negative accuracy does not reach the level achieved on standard negatives, and Figure~\ref{fig:hard_neg_acc} shows that gains in hard negative accuracy come at the cost of feasible claim acceptance. This tradeoff is not resolved by any prompting condition: stricter verification reduces over-acceptance of compositionally infeasible claims but increases rejection of feasible ones, because models cannot yet reliably distinguish the two through 
compositional reasoning at any threshold. Prompting shifts where models operate on the curve, but the curve itself reflects a reasoning bottleneck that verification strategy guidance alone cannot overcome~\citep{arditi2022refusal, su2025ai}.


\section{Reasoning Trace Analysis}
\label{sec:judge}

To complement our outcome-based evaluation with trace-level evidence~\citep{lightman2023let}, we conduct a supplementary, single-domain analysis on NLI4CT. We score each reasoning trace using a rubric-based judge~\citep{kim2023prometheus} targeting salient-constraint checking versus full CWA verification. Each of four dimensions is scored from $-2$ (salient-constraint checking) to $+2$ (full CWA) and aggregated to 0--100. Higher scores indicate reasoning closer to full CWA verification. The complete rubric and human validation results are detailed in Appendix~\ref{app:rubric}.

\begin{table}[t]
\centering
\small
\begin{tabular}{@{}lcccccc@{}}
\toprule
& \multicolumn{3}{c}{\textbf{Standard Negative}}
& \multicolumn{3}{c}{\textbf{Adversarial Neg}} \\
\cmidrule(lr){2-4}\cmidrule(lr){5-7}
& Correct (TN)$\uparrow$ & Incorrect (FN)$\uparrow$ & Avg$\uparrow$
& TN$\uparrow$ & FN$\uparrow$ & Avg$\uparrow$ \\
\midrule
GPT-5-mini       & 86.2 & 61.7 & 84.0 & 92.3 & 64.3 & 88.0 \\
Gemini-3.0-Flash    & 83.0 & 61.9 & 80.1 & 91.4 & 65.1 & 84.9 \\
~~$+$OWA        & 87.7 & 59.2 & 83.6 & 92.0 & 63.6 & 84.5 \\
~~$+$OWA        & 88.0 & 63.2 & 85.9 & 93.9 & 63.4 & 87.1 \\
Claude-Sonnet-4.5      & 89.1 & 71.4 & 87.6 & 92.1 & 70.9 & 89.2 \\
\bottomrule
\end{tabular}
\caption{CWA trace scores (0--100) on NLI4CT. \textbf{Correct (TN)} denotes traces where the model successfully rejected the claim (True Negative); \textbf{Incorrect (FN)} denotes traces where the model falsely accepted it (False Negative). $+$CWA and $+$OWA are prompt intervention.}
\label{tab:judge}
\end{table}

\textbf {Cross-model comparison. }
CWA scores independently recover the robustness ranking from
Table~\ref{tab:routing}: Gemini-3.0-Flash scores lowest (80.1 standard, 84.9
adversarial), GPT-5-mini intermediate (84.0, 88.0), and Claude-Sonnet-4.5
highest (87.6, 89.2). The rubric-based ordering matches the
accuracy-based ordering, confirming that robustness to compositional
falsification tracks how thoroughly models check constraints in their
reasoning traces.

\textbf{Within-model comparison. }
Across all models, average scores on compositionally infeasible claims
are higher than on standard claims (Gemini-3.0-Flash 84.9 vs.\ 80.1, GPT-5-mini
88.0 vs.\ 84.0, Claude-Sonnet-4.5 89.2 vs.\ 87.6), indicating that the
compositional nature of these claims induces more thorough constraint
checking in reasoning traces. The gap between correct and incorrect
traces is larger on compositionally infeasible claims than on standard
ones, because correctly rejecting a compositionally infeasible claim
requires enumerating and composing non-salient constraints, raising the
reasoning bar, while failed traces on both claim types reflect a similar
pattern of early termination at the salient constraint.

\textbf{Intervention. }
CWA prompting raises correct case scores on standard
($83.0 \to 88.0$) and adversarial ($91.4 \to 93.9$) claims, confirming
CWA guidance shifts reasoning toward more thorough constraint
checking. Incorrect adversarial scores remain nearly unchanged
($65.1 \to 63.4$), indicating a reasoning threshold below which
incorrect verdicts occur. OWA prompting does not consistently lower
scores despite increasing acceptance rate, confirming the rubric measures
constraint coverage in reasoning rather than verdict direction.

\section{Discussion}
\label{sec:discussion}

\textbf{Verification mode in self-verification.  }
Scientific claim verification is a form of external verification as it evaluates external claims against evidence. Verification also arises internally, when models check their own reasoning as a common reasoning strategy. If salient-constraint checking reflects a general verification tendency rather than a property specific to external claims, the same failure mode should appear in self-verification. Models may accept their own incorrect responses without checking all constraints that would reveal the error.

We test this by asking GPT-4o-mini to self-verify its own responses on 200 problems each from AIME ~\citep{aime_1983_2024} and GPQA Diamond ~\citep{rein2024gpqa} under baseline and CWA/OWA prompt intervention, with a stronger model judging both verdict and reasoning validity. As shown in Table~\ref{tab:selfverif}, baseline performance 
falls between OWA and CWA on both $F_1$ (over valid rejections and valid acceptances) and Precision at Rejection (Prec@Rej, defined as the proportion of correct rejections accompanied by faithful reasoning), mirroring the ROC trajectory in \S\ref{sec:routing}. OWA underperforms baseline, consistent with increased over-acceptance. CWA achieves the highest Prec@Rej, indicating that stricter constraint checking improves not just rejection rate but the quality of reasoning behind rejections, which could be a potential direction for fixing confirmation bias ~\citep{huang2023large, zhou2025variation, wan2025unveiling}.

\begin{table}[t]
\centering
\small
\begin{tabular}{@{}lcccc@{}}
\toprule
& \multicolumn{2}{c}{\textbf{GPQA}}
& \multicolumn{2}{c}{\textbf{AIME}} \\
\cmidrule(lr){2-3}\cmidrule(lr){4-5}
& $F_1\uparrow$ & Prec@Rej$\uparrow$ & $F_1\uparrow$ & Prec@Rej$\uparrow$ \\
\midrule
Baseline & .218 & .253 & .287 & .256 \\
OWA      & .165 & .206 & .238 & .226 \\
CWA      & \textbf{.262} & \textbf{.273}
         & \textbf{.347} & \textbf{.290} \\
\bottomrule
\end{tabular}
\caption{Self-verification on GPQA Diamond and AIME. A rejection is valid if the solver response is incorrect and the verifier reasoning correctly identifies the exact solver error. $F_1$: macro-$F_1$ 
over valid rejections and valid acceptances. Prec@Rej: precision over valid rejections only. \textbf{Bold}: best.}
\label{tab:selfverif}
\end{table}

\textbf{Origins of salient-constraint checking.  }
We hypothesize that salient-constraint checking emerges from pretraining on scientific data, where claims are typically supported by directly retrievable features and compositional violations rarely appear. In this setting, checking the salient constraint is nearly always sufficient, providing little training signal to develop full CWA verification~\citep{dziri2023faith, kambhampati2024position}. This mismatch between pretraining corpora and the structure of scientific misinformation explains both why salient-constraint checking is the default and why prompting can shift but not overcome the reasoning bottleneck.

\textbf{Future work.  }
This work is primarily diagnostic, as we identify and isolate a failure mode in scientific verification. Although our graph-based pipeline yields structured reasoning trajectories suitable for post-training (SFT and RLVR), we do not view this as a straightforward solution~\citep{yue2025does}. We trained Qwen3-8B on compositional examples using both approaches and shifted the decision boundary toward broader rejection rather than improving compositional reasoning quality. More fundamentally, most training tasks reward $\exists$-type reasoning, where finding any supporting evidence is sufficient for success, leaving the $\forall$-type reasoning that full CWA verification requires systematically undertrained~\citep{dziri2023faith, mccoy2023embers}. Addressing both the compositional inference bottleneck and this verification tendency is likely necessary for genuine progress, and we leave the design of such interventions to future work.
\section{Related Work}

\paragraph{Scientific claim verification.}
Claim verification benchmarks span general knowledge 
\citep{thorne2018fever, aly2021fact}, scientific literature 
\citep{wadden2020fact, wadden2022scifact, lal2025musciclaims, kumar2025sciclaimhunt}, clinical trials 
\citep{jullien2023nli4ct, jullien2024semeval}, tables 
\citep{chen2019tabfact, lu2023scitab}, and charts 
\citep{akhtar2024chartcheck, wang2025sciver}. Previous works improve claim verification ability of LLM through retrieval ~\citep{adjali2024exploring, liu2024retrieval,chen2024complex, sallami2025claimveragents}, decomposition~\citep{sahu2024pelican, lu2025optimizing, vladika2025step} and reasoning improvement~\citep{wadden2025sciriff, freedman2025argumentative}. Models are known to 
exploit surface heuristics rather than genuine reasoning 
\citep{gururangan2018annotation, clark2019don, 
niven2019probing, mccoy2019right}. We identify why existing benchmarks cannot detect a specific failure mode, and construct examples that can detect it.

\paragraph{Compositional adversarial data construction.}
Compositional reasoning benchmarks use graph-based methods to generate multi-hop question answering pairs \citep{yang2018hotpotqa, tu-etal-2019-multi, pan-etal-2020-semantic, trivedi-etal-2022-musique}. Previous works construct adversarial and counterfactual data to evaluate LLMs through human annotation \citep{gardner2020evaluating, kiela-etal-2021-dynabench, liu2023evaluating}, heuristic rule-based perturbations \citep{gan2024reasoning, waghela2024adversarial}, and automated LLM-driven synthetic generation \citep{wu2024reasoning, bhattacharjee2024zero, zhang2025falsereject}. We follow the line of work that applies systematic graph perturbations to generate adversarial data\citep{sharma2023temporal, wu2024reasoning, zhang2024darg, hong2025evaluating, cheng2025understanding}.

\newpage 
\bibliography{references}
\bibliographystyle{colm2026_conference}

\appendix

\section{Benchmark Collection}
\label{app:datasets}
\begin{table}[h]
\centering\small
\renewcommand{\arraystretch}{1.2}
\begin{tabular}{@{}p{1cm}p{1.5cm}p{1.8cm}p{1.2cm}p{1.2cm}p{5cm}@{}}
\toprule
\textbf{Dataset} & \textbf{Original source} & \textbf{Subset used} & \textbf{Generated} & \textbf{Validated} & \textbf{Notes} \\
\midrule
\textbf{NLI4CT}
& Clinical trial claims
& Single-trial examples only
& 242
& 203
& We exclude multi-trial claims and generate hard negatives only from single-trial evidence. \\[4pt]

\textbf{SciTab}
& Scientific table claims
& Randomly selected table subset
& 230
& 200
& Candidate pool was sufficient to retain 200 validated hard negatives after manual filtering. \\[4pt]

\textbf{SciVer}
& Scientific figure claims
& Single-chart direct-reasoning subset
& 213
& 180
& SciVer includes multiple reasoning types and multimodal evidence settings (e.g., chart, text, table, or multiple evidence pieces). We restrict to single-chart examples, which yields 213 source items; after validation, fewer than 200 hard negatives remain. \\
\bottomrule
\end{tabular}
\caption{Dataset subsets and annotation outcomes for adversarial hard-negative construction.}
\label{tab:dataset_stats}
\end{table}
\begin{table}[h]
\centering\small
\resizebox{\linewidth}{!}{%
\renewcommand{\arraystretch}{1.25}
\begin{tabular}{@{}
  >{\raggedright\arraybackslash}p{1.35cm}
  >{\raggedright\arraybackslash}p{3.5cm}
  >{\raggedright\arraybackslash}p{3.7cm}
  >{\raggedright\arraybackslash}p{3.8cm}
  >{\raggedright\arraybackslash}p{5.2cm}@{}}
\toprule
& \textbf{Evidence} & \textbf{Feasible claims} & \textbf{Standard infeasible claims} & \textbf{Our hard-negative perturbations} \\
\midrule

\textbf{NLI4CT}
& Breast cancer clinical trial report excerpts; single section per instance.
& Expert-written claims requiring biomedical, numerical, and commonsense reasoning.
& Contrast-set perturbations, including numerical, lexical, semantic, and syntactic changes.
& Compositional clinical misreadings such as qualifier omission, scope/population shifts, conditional or boundary flips, cross-arm or proxy misattribution, and surfaced assumptions. \\[5pt]

\textbf{SciTab}
& Scientific tables with captions from NLP/ML papers.
& Sentences from the paper grounded in the table; two-annotator verified.
& InstructGPT-generated negatives with wrong calculations, value mismatches, and incorrect approximations.
& Structural table perturbations such as inconsistent baselines, part--whole and averaging errors, derived-metric overextension, null-as-zero aggregation, and occasional causal or categorical overreach. \\[5pt]

\textbf{SciVer}
& Chart screenshots from CS/arXiv papers; direct-reasoning subset only.
& Expert-written claims grounded in chart content.
& Expert-rewritten negatives with hallucinated values, value misattribution, or relationship reversal.
& Structural visual misreadings such as local-to-global overreach, ignored panel or layout boundaries, visual analogy or element conflation, prominence-based over-interpretation, and unsupported comparison classes. \\

\bottomrule
\end{tabular}}
\caption{Benchmark datasets used in this work, including both their original infeasible-claim construction and our compositional hard-negative perturbations.}
\label{tab:benchmarks}
\end{table}

\section{Annotation Details}
\label{app:annotate}
Validation for this work is diagnostic rather than benchmark construction. Our goal is to confirm that generated hard negatives are not accidentally feasible and do not introduce ambiguous cases, not to establish a reusable labeled dataset with publication-grade inter-annotator agreement. In this context, single-annotator validation is standard practice as the annotator's role is to apply a well-defined binary judgment (is this claim clearly infeasible given the document, or is it ambiguous/accidentally true?) rather than to make subjective interpretive assessments.

All generated hard negatives were manually validated by a single author. During annotation, the annotator was given the source evidence, the generated reasoning graph, the perturbation logic (including the corrupted node), and the final infeasible claim. Each example was judged based on both whether the final claim was clearly infeasible with respect to the source and whether the intermediate corruption constituted a valid adversarial transformation.

The final acceptance rates were 203/242 (83.9\%) for NLI4CT, 200/230 (87.0\%) for SciTab, and 180/213 (84.5\%) for SciVer. Most rejected candidates were ambiguous rather than incorrect. They introduced higher-level interpretations, external assumptions, or broader statistical conclusions that were not clearly licensed by the source, but were also not unambiguously contradicted by it. A smaller portion were genuinely invalid generations, typically because an intended overgeneralization remained accidentally compatible with the evidence.

Validation was conducted by a single annotator with CS-domain expertise. SCITAB and SciVer draw exclusively from CS papers, making the annotator's domain knowledge sufficient and inter-annotator variability is expected to be low. For NLI4CT, judgments were grounded in source-document consistency rather than independent clinical expertise. The task was to determine whether the generated claim and corruption were supported by the provided trial text, not to assess medical validity beyond the evidence. This grounds the judgment in objective textual entailment rather than domain-specific interpretation, limiting the role of annotator background knowledge. We acknowledge that multi-annotator review with inter-rater agreement reporting would further strengthen label quality for future benchmark use, and we intend to address this in future work.

\section{Decoding Randomness}
To assess whether findings are stable across decoding randomness, 
we evaluate the Qwen3 family at three temperatures: $T=0$, $T=0.6$, 
and $T=1.0$ on SCITAB. Figure~\ref{fig:temperature} shows accuracy on positive, 
standard negative, and adversarial negative claims across model sizes 
for each temperature.
\begin{figure}[h]
    \centering
    \includegraphics[width=1\linewidth]{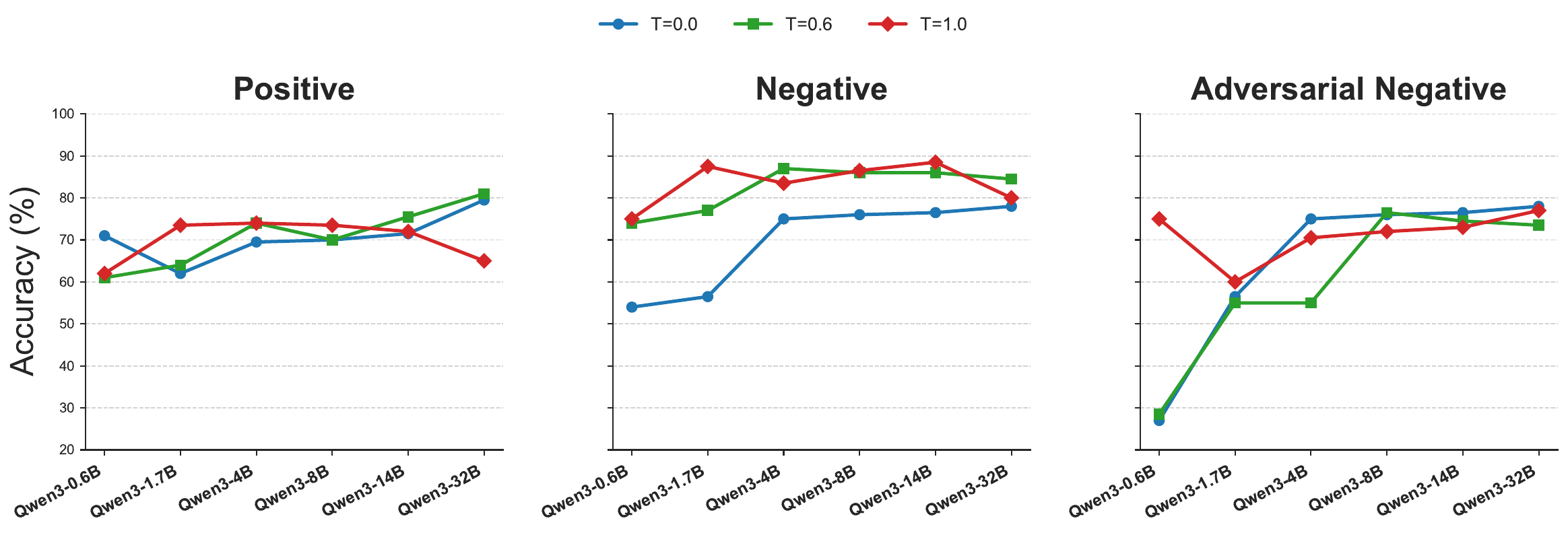}
    \caption{Qwen3 model families performance of SCITAB with decoding at different temperatures}
    \label{fig:temperature}
\end{figure}
Adversarial negative accuracy is lower than standard negative accuracy 
at every model size and every temperature, confirming that the routing 
failure is not an artifact of a particular decoding choice. However, 
the magnitude of the gap varies substantially with temperature. At 
$T=0$, Qwen3-0.6B achieves only 27\% on adversarial negatives despite 86.5\% on standard 
negatives, a 59.5-point difference, because deterministic decoding 
commits fully to the model's prior, which at small scale is 
strong acceptance. As scale increases, the gap narrows 
monotonically, reaching $-8.0\%$ at 32B, producing the clearest 
scaling signal. At $T=1.0$, the gap at small scale nearly vanishes: 
Qwen3-0.6B shows 0\% difference, because high-temperature sampling 
occasionally follows lower-probability reasoning paths that 
accidentally produce correct rejections. This stochastic masking 
makes $T=1.0$ a less reliable diagnostic for the routing failure at 
small scale. $T=0.6$ reaches a balance between these extremes, 
producing stable and monotonically improving adversarial negative 
accuracy across scale without the extreme compression of $T=1.0$ or 
the floor effects of $T=0$ at small sizes. We therefore use $T=0.6$ 
as our primary evaluation setting for the Qwen3 scaling experiment, and report full results across all three temperatures here for completeness.

\section{Ablation Study}
\label{sec:ablation}


\paragraph{Surface form.}
Compositionally infeasible claims tend to be longer than standard infeasible claims due to the compositionality. If surface form rather than compositional structure drives the gap, rephrasing standard infeasible claims to match the style of compositionally infeasible claims should produce a similar drop. We take a subset of 50 standard infeasible claims from each domain, and rephrase each standard infeasible claim into a longer, more formal surface form while preserving the identical factual assertion and violation. As shown in Figure~\ref{fig:surface}, rephrased negatives track original standard negatives closely across all three domains, which proves that surface form does not explain the gap between standard infeasible and compositionally infeasible claims. 
\begin{figure}[h]
    \centering
    \includegraphics[width=\linewidth]{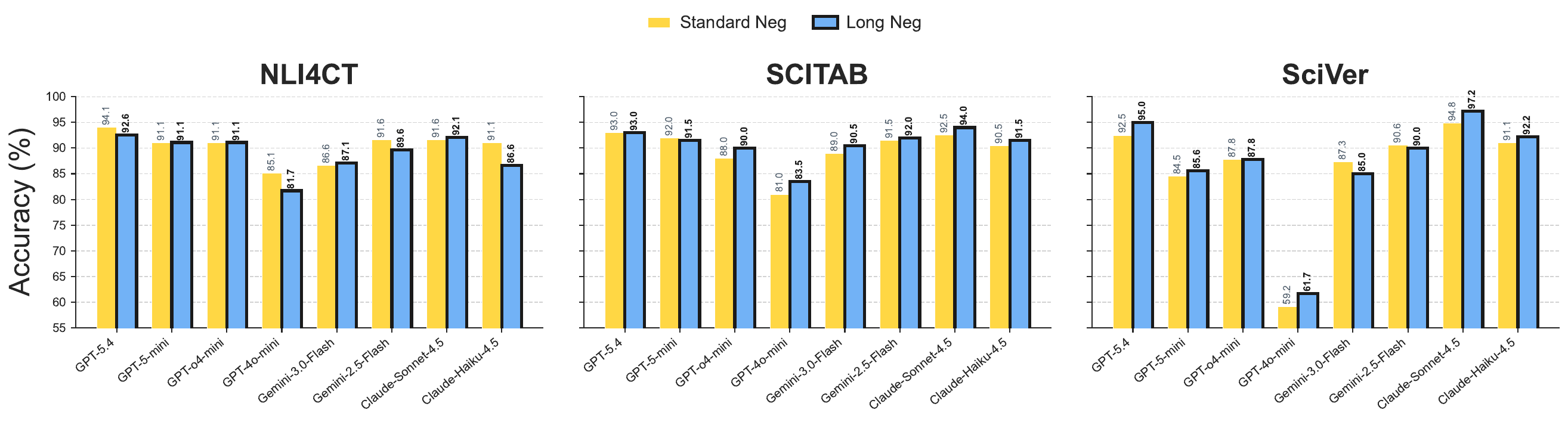}
    \caption{Longer negatives show no improvement or degradation over standard negatives, ruling out surface form as the reason for performance drop on adversarial negatives}
    \label{fig:surface}
\end{figure}

\paragraph{Generator-Verifier Bias}
To assess whether evaluation outcomes are affected by the stylistic artifacts 
introduced by the generation pipeline, we compare performance on two claim versions 
across 50 adversarial negatives per domain. In the \textbf{rephrased} 
condition, GPT-5-mini-generated claims are rephrased by 
Gemini-3.0-Flash while preserving logical content; in the 
\textbf{original} condition, raw GPT-5-mini-generated claims 
are evaluated without rephrasing. If GPT-5-mini were biased 
toward its own writing style, it should score consistently 
higher on original claims than on rephrased ones. 
Table~\ref{tab:gv-bias} shows no consistent pattern supporting 
this. On NLI4CT, both evaluators produce identical scores across 
conditions, showing no effect. On SCITAB, GPT-5-mini scores 
higher on original claims (82\% vs.\ 74\%), consistent with a 
weak stylistic preference. On SciVer, GPT-5-mini scores higher on rephrased 
claims (66\% vs.\ 78\%), which is opposite to the bias direction. 
GPT-5-mini appears more sensitive to surface-level style 
changes than Gemini-3.0-Flash, but the inconsistency in 
direction across domains and evaluators indicates no systematic 
generation-verifier bias. Therefore, the results reported in 
the main paper for the GPT model family are not 
inflated by stylistic familiarity with the generation pipeline.

\begin{table}[h]
\centering
\small
\begin{tabular}{llrr}
\toprule
\textbf{Domain} & \textbf{Evaluator} & 
\textbf{Rephrased (\%)} & \textbf{Original (\%)} \\
\midrule
NLI4CT  & GPT-5-mini        & 86.0 & 86.0 \\
        & Gemini-3.0-Flash  & 76.0 & 76.0 \\
\midrule
SCITAB  & GPT-5-mini        & 74.0 & 82.0 \\
        & Gemini-3.0-Flash  & 72.0 & 70.0 \\
\midrule
SciVer  & GPT-5-mini        & 78.0 & 66.0 \\
        & Gemini-3.0-Flash  & 66.0 & 70.0 \\
\bottomrule
\end{tabular}
\caption{Hard negative accuracy on Gemini-rephrased claims 
(rephrased) vs.\ raw GPT-5-mini-generated claims (original), 
evaluated by two independent model families. Differences are 
small and inconsistent in direction, ruling out systematic 
generation-verifier bias.}
\label{tab:gv-bias}
\end{table}
```
\section{LLM Judge Rubrics and Human Agreement}
\label{app:rubric}

\begin{table}[H]
\centering
\small
\resizebox{\linewidth}{!}{%
\begin{tabular}{@{}>{\raggedright\arraybackslash}p{3.4cm}
                >{\raggedright\arraybackslash}p{6cm}
                >{\raggedright\arraybackslash}p{6cm}@{}}
\toprule
\textbf{Dimension}
    & \textbf{Salient-constraint checking ($-2$)}
    & \textbf{Full CWA ($+2$)} \\
\midrule
Constraint enumeration
    & Checks one prominent constraint
    & Lists all constraints before checking any \\[3pt]
Non-salient coverage
    & Stops at salient constraint
    & Checks non-salient and composed constraints \\[3pt]
Evidence boundary
    & Accepts on local consistency
    & Rejects when scope exceeds evidence \\[3pt]
Verdict warrant
    & Notes gaps but accepts
    & Verdict follows from all gaps \\
\bottomrule
\end{tabular}}
\caption{Rubric dimensions for scoring reasoning traces on the salient-constraint checking vs.\ full CWA axis. Each dimension is scored $\{-2,\ldots,+2\}$ with anchored examples.}
\label{tab:rubric}
\end{table}

To assess the validity of the rubric-based LLM judge, a single annotator with NLP research experience independently scored a stratified sample of 50 reasoning traces from NLI4CT, drawn to cover correct rejection and  incorrect rejections in proportion to their occurrence in the full evaluation set. The annotator scored each trace on the same four dimensions ($-2$ to $+2$ per dimension, aggregated to 0--100) without access to the LLM judge's scores. Agreement between the human annotator and LLM judge, measured as Spearman's $\rho$ on the aggregated 0--100 scale, was $\rho = 0.43$ for  correct traces and $\rho = 0.53$ for incorrect traces, yielding a mean of $\rho = 0.47$ across the two models. This indicates moderate agreement between the human annotator and the automated judge. The overall pooled correlation is lower due to range restriction, as most traces score in a narrow band (75--90), which compresses rank differences and reduces Spearman's sensitivity to genuine agreement. Single-annotator validation is a limitation of this study. Multi-annotator validation with formal inter-annotator agreement reporting remains an important direction for future work.

\section{Self-Verification Experiment Details}
\paragraph{Task and Dataset Selection}
We evaluate on GPQA Diamond~\citep{rein2024gpqa} and 
AIME~\citep{aime_1983_2024}. Both datasets require multi-step 
reasoning where a correct solution must satisfy multiple constraints 
simultaneously, making them natural analogues to the compositional 
verification setting in the main paper. We select GPT-4o-mini as the 
solver because it achieves non-trivial but imperfect accuracy on both 
tasks, producing a mix of correct and incorrect responses suitable for 
studying self-verification behavior. A model that always succeeds 
provides no incorrect responses to verify; a model that always fails 
provides no correct responses to accept. GPT-4o-mini's accuracy in the 
range of 25--35\% on these tasks provides sufficient signal in both 
directions.

\paragraph{Experiments} For each problem, GPT-4o-mini first produces a solution in a single 
forward pass. The same model then self-verifies its own response under 
one of three prompting conditions: baseline, CWA-instructed, or 
OWA-instructed, using the prompts similar to Appendix G. The 
self-verification step receives the problem statement and the model's 
own solution and must output (1) accept or reject (2) reasoning.

Ground truth verdicts and reasoning quality are assessed by 
GPT-5-mini, which achieves over 90\% accuracy on both benchmarks and 
serves as a reliable judge. The judge evaluates two things independently: 
(1) whether the solver's accept/reject answer is correct, and (2) whether the 
verifier's reasoning correctly identifies the source of error when 
rejecting.

\paragraph{Evaluation Metrics} we define the following case taxonomy:

\begin{table}[h]
\centering
\small
\begin{tabular}{p{1.15cm}p{1.2cm}p{1.35cm}p{1.0cm}p{5.0cm}p{1.45cm}}
\toprule
\textbf{Solver} & \textbf{Verifier} & \textbf{Reasoning} & \textbf{Label} & \textbf{Description} & \textbf{Count as} \\
\midrule
Incorrect & Reject & Valid   & TP             & Verifier correctly identifies the actual error & True Positive \\
Correct   & Accept & ---     & TN             & Verifier correctly accepts a correct solution  & True Negative \\
Incorrect & Reject & Invalid & Lucky Catch    & Verifier rejects but for wrong/unrelated reason & False Positive \\
Correct   & Reject & ---     & FP             & Verifier incorrectly rejects a correct solution & False Positive \\
Incorrect & Accept & ---     & FN             & Verifier incorrectly accepts an incorrect solution & False Negative \\
\bottomrule
\end{tabular}
\caption{Case taxonomy for self-verification evaluation. }
\label{tab:selfverif-taxonomy}
\end{table}

We report macro-F1 and Precision at Rejection (Prec@Rej) defined as TP / (TP + lucky\_catch), measuring what fraction of rejections reflect genuine constraint checking rather than incidental rejection.

\subsection*{Results}

\begin{table}[t]
\centering
\small
\begin{tabular}{llrrrrrrrr}
\toprule
\textbf{Dataset} & \textbf{Variant} & \textbf{n} & \textbf{Acc} & \textbf{F1} & \textbf{TP} &
\textbf{TN} & \textbf{FP} & \textbf{FN} & \textbf{Prec@Rej} \\
\midrule
GPQA & baseline & 198 & 0.313 & 0.218 & 19 & 43 & 82 & 54 & 19/75 $\approx$ 0.253 \\
     & CWA      & 198 & \textbf{0.318} & \textbf{0.262} & 24 & 39 & 94 & 41 & \textbf{24/88 $\approx$ 0.273} \\
     & OWA      & 198 & 0.283 & 0.165 & 14 & 42 & 81 & 61 & 14/68 $\approx$ 0.206 \\
\midrule
AIME & baseline & 200 & 0.255 & 0.287 & 30 & 21 & 93 & 56 & 30/117 $\approx$ 0.256 \\
     & CWA      & 200 & \textbf{0.285} & \textbf{0.347} & 38 & 19 & 101 & 42 & \textbf{38/131 $\approx$ 0.290} \\
     & OWA      & 200 & 0.230 & 0.238 & 24 & 22 & 87 & 67 & 24/106 $\approx$ 0.226 \\
\bottomrule
\end{tabular}
\end{table}

The ordering  CWA $>$ baseline $>$ OWA on both F1 and Prec@Rej 
mirrors the ROC trajectory observed in the main paper, 
with the baseline lying between OWA and CWA. This replicates the 
external verification pattern in an internal setting, consistent with 
salient-constraint checking being a general verification tendency rather 
than a property specific to external evidence evaluation.

CWA prompting improves Prec@Rej on both datasets, indicating that 
stricter constraint checking improves not just rejection rate but the 
quality of reasoning behind rejections. The real catch rate under CWA 
(0.273 on GPQA, 0.290 on AIME) exceeds the baseline (0.253, 0.256), 
confirming that CWA guidance produces more genuine error identification 
rather than merely more rejections.

However, the absolute F1 values remain low across all conditions. 
This reflects the same reasoning bottleneck identified that
prompting shifts the operating point but cannot overcome the 
underlying compositional reasoning limitation. The high false positive 
counts under CWA (94 on GPQA, 101 on AIME) indicate that stricter 
verification causes over-rejection of correct solutions, consistent 
with the tradeoff between acceptance and rejection accuracy 
observed in the external verification setting.
\newpage
\section{Prompts for Data Generation}
\label{app:gen prompt}
\begin{figure}[H]
{\centering\large\textbf{Graph Generation Prompt}\par\smallskip}
\begin{tcolorbox}[
  enhanced, breakable,
  colback=gray!4, colframe=gray!45,
  boxrule=0.6pt, arc=3pt,
  left=8pt, right=8pt, top=6pt, bottom=6pt,
  width=\linewidth
]

\noindent Build a reasoning graph for claim verification. You are given:
\begin{itemize}[topsep=1pt,itemsep=0pt,parsep=0pt,leftmargin=*]
  \item \clinic{clinical trial: a full document, a target excerpt, and a feasible claim}
  \item \tabdom{table: a table and a feasible claim}
  \item \chartdom{chart: an image and a feasible claim}
\end{itemize}
The claim is the surface. Your job is to write TRUE context beneath it so later we can create subtle misinformation by mischaracterizing that context. Write three layers:

\tcblower

\textbf{LAYER 1 --- OBSERVATIONS}\quad\textit{(CLAIM-ANCHORING O1..O5: elements related to or directly supporting the claim)}
\begin{itemize}[topsep=1pt,itemsep=0pt,parsep=0pt,leftmargin=*]
  \item \clinic{clinical trial: include numerical values, cohort/arm labels, outcome definitions, denominators, eligibility criteria, measurement definitions, timeframes.}
  \item \tabdom{table: describe structural regions or bindings (rows, columns, headers); focus on distributions, hierarchy, or constraints; avoid isolated cells unless they define outliers.}
  \item \chartdom{chart: describe visible anchors (legend, axes, ticks, panels, marks); include positions, scales, labels, and annotated elements.}
\end{itemize}
\textit{Rules:} all observations must be literal and directly present in the input; no arithmetic, comparison, or inference.

\medskip
\textbf{LAYER 2 --- CONTEXT}\quad\textit{(describe how observations relate structurally)}
\begin{itemize}[topsep=1pt,itemsep=0pt,parsep=0pt,leftmargin=*]
  \item \clinic{clinical trial: connect observations to broader trial structure.}
  \item \tabdom{table: describe structural relationships (ordering, gaps, trends, scope boundaries).}
  \item \chartdom{chart: describe structural relationships (ordering, gaps, trends, scope boundaries).}
\end{itemize}
\textit{Rules:} each context step must cite $\geq$2 observations; must be directly supported by the input; do NOT restate the claim.

\medskip
\textbf{LAYER 3 --- INTERPRETATION}\quad\textit{(explain what the contextual structure implies)}
\begin{itemize}[topsep=1pt,itemsep=0pt,parsep=0pt,leftmargin=*]
  \item \clinic{clinical trial: ground in full document; focus on population, endpoints, conditions, or applicability.}
  \item \tabdom{table: focus on ranking scope, comparison validity, ordering, representativeness.}
  \item \chartdom{chart: focus on trends, generalization, stability, or what values represent.}
\end{itemize}
\textit{Rules:} each interpretation step must cite $\geq$2 context steps; must remain grounded in the input; do NOT restate the claim.

\end{tcolorbox}
\end{figure}
\newpage
\noindent{\large\textbf{Infeasible Claim Generation Prompt}}
\vspace{3pt}

\begin{tcolorbox}[
  enhanced,
  breakable,
  colback=gray!4,
  colframe=gray!45,
  boxrule=0.5pt,
  arc=2pt,
  left=5pt,right=5pt,top=4pt,bottom=4pt,
  boxsep=0pt,
  width=\linewidth,
  before skip=0pt,
  after skip=0pt
]
\footnotesize

\textbf{Input.}
Generate a plausible-but-false (\textbf{infeasible}) claim.
You are given:
\clinic{ clinical trial: full document, target excerpt, source claim;}
\tabdom{ table: table, source claim;}
\chartdom{ chart: image, source claim;}
and a reasoning graph with nodes
\textbf{O} (observations: concrete facts),
\textbf{C} (context: structural relationships), and
\textbf{I} (interpretation: how structure supports conclusions).

\vspace{0.25em}
\textbf{Core idea.}
\emph{Preserve observations:} keep all entities, values, labels, and visible evidence unchanged; the claim must remain grounded in the same observations as the source claim.
\emph{Corrupt reasoning:} introduce a minimal but incorrect structural interpretation; the error must come from how evidence is combined, not what is observed; the claim should remain locally plausible.

\vspace{0.25em}
\textbf{Structural falsifiability.}
The claim must
(1) not be falsifiable by checking a single value, cell, or sentence;
(2) only be falsifiable by combining multiple pieces of evidence; and
(3) require integrating different regions, passages, or structures.

\vspace{0.25em}
\textbf{Domain-specific corruption.}
\clinic{ Clinical trial: errors arise from scope, conditions, populations, or trial structure.}
\tabdom{ Table: errors arise from structural relationships such as grouping, units, or aggregation, and should involve multiple regions (e.g., row group + footnote + column).}
\chartdom{ Chart: errors arise from visual structure such as scope, trend, grouping, or segmentation, and should involve multiple regions or patterns in the figure.}

\vspace{0.25em}
\textbf{What to do.}
(A) Choose exactly one context (\textbf{C}) or interpretation (\textbf{I}) node to corrupt.
(B) Write \texttt{corrupted\_node\_text} as a minimal change to the original structure; keep observations unchanged and make the error subtle and plausible.
(C) Write \texttt{propagation\_notes} in 1--2 sentences explaining how the corrupted structure leads to the false claim.
(D) Write \texttt{infeasible\_claim} in natural, domain-appropriate language; reuse the same observations as the source claim; do not mention counterevidence; do not rely on exact numeric contradictions.

\vspace{0.25em}
\textbf{Self-check (required).}
The claim must be
(1) locally plausible using the same evidence as the source claim;
(2) not falsifiable from a single observation or sentence;
(3) refutable only by combining multiple pieces of evidence;
(4) free of new entities, values, or external knowledge; and
(5) not obviously false or ambiguous to a careful reader.
If any condition fails, revise.

\vspace{0.25em}
\textbf{Disallowed.}
Changing or fabricating observations; explicit contradictions (numeric flip, label swap); causal or normative claims; external knowledge not in the input; trivial or single-step falsification.

\end{tcolorbox}

\section{Prompts for CWA/OWA Intervention}
\begin{figure}[H]
{\centering\large\textbf{Evaluation Prompt}\par\smallskip}
\begin{tcolorbox}[
  enhanced, breakable,
  colback=gray!4, colframe=gray!45,
  boxrule=0.6pt, arc=3pt,
  left=8pt, right=8pt, top=6pt, bottom=6pt,
  width=\linewidth
]
\small

\noindent Determine whether a claim is \textbf{FEASIBLE} or \textbf{INFEASIBLE} given evidence.

\medskip
\noindent You are given:
\begin{itemize}[topsep=1pt,itemsep=0pt,parsep=0pt,leftmargin=*]
  \item \clinic{clinical: full trial + target excerpt}
  \item \tabdom{table: table context}
  \item \chartdom{chart: image (+ optional context)}
  \item claim
\end{itemize}

\tcblower

\noindent\colorbox{blue!10}{\strut\textbf{Type 1: Baseline}}

\smallskip
\noindent Use the evidence to decide whether the claim is supported.

\medskip
\noindent\colorbox{cyan!12}{\strut\textbf{Type 2: OWA}}

\smallskip
\noindent You are under open-world assumption. Use the evidence to decide whether the claim is supported. Do not reject a claim based on missing evidence alone.

\medskip
\noindent\colorbox{orange!14}{\strut\textbf{Type 3: CWA}}

\smallskip
\noindent You are under closed-world assumption.
\begin{itemize}[topsep=1pt,itemsep=0pt,parsep=0pt,leftmargin=*]
  \item \textbf{FEASIBLE} only if all parts of the claim are explicitly supported.
  \item \textbf{INFEASIBLE} if any part is unsupported or contradicted.
\end{itemize}
\noindent Absence of supporting evidence counts as \textbf{INFEASIBLE}.

\medskip
\noindent\colorbox{green!12}{\strut\textbf{Type 4: DCP+OWA}}

\smallskip
\noindent You are under open-world assumption with decomposition.

\noindent\textbf{STEP 1 --- DECOMPOSE}\\
Break the claim into minimal subclaims required for it to hold.

\noindent\textbf{STEP 2 --- CHECK CONTRADICTION}\\
For each subclaim:
\begin{itemize}[topsep=1pt,itemsep=0pt,parsep=0pt,leftmargin=*]
  \item mark \textbf{INFEASIBLE} only if explicitly contradicted
  \item absence of evidence does not invalidate it
\end{itemize}

\noindent\textbf{STEP 3 --- VERDICT}
\begin{itemize}[topsep=1pt,itemsep=0pt,parsep=0pt,leftmargin=*]
  \item \textbf{FEASIBLE}: no subclaim is contradicted
  \item \textbf{INFEASIBLE}: at least one subclaim is contradicted
\end{itemize}

\medskip
\noindent\colorbox{red!10}{\strut\textbf{Type 5: DCP+CWA}}

\smallskip
\noindent You are under closed-world assumption with decomposition.

\noindent\textbf{STEP 1 --- DECOMPOSE}\\
Break the claim into minimal subclaims required for it to hold.

\noindent\textbf{STEP 2 --- CHECK SUPPORT}\\
For each subclaim:
\begin{itemize}[topsep=1pt,itemsep=0pt,parsep=0pt,leftmargin=*]
  \item it must be fully supported by the evidence
  \item unsupported or contradicted subclaims fail
\end{itemize}

\noindent\textbf{STEP 3 --- VERDICT}
\begin{itemize}[topsep=1pt,itemsep=0pt,parsep=0pt,leftmargin=*]
  \item \textbf{FEASIBLE}: all subclaims are supported
  \item \textbf{INFEASIBLE}: any subclaim is not supported
\end{itemize}

\medskip
\noindent\textbf{Output format}
\begin{itemize}[topsep=1pt,itemsep=0pt,parsep=0pt,leftmargin=*]
  \item (\texttt{DCP+OWA} and \texttt{DCP+CWA}) \texttt{Subclaims}: list
  \item \texttt{Reason}: brief explanation grounded in evidence
  \item \texttt{Verdict}: \textbf{FEASIBLE} or \textbf{INFEASIBLE}
\end{itemize}

\end{tcolorbox}
\end{figure}

\section{Data Generation Examples}
\label{app:gen example}
\begin{figure*}[h]
\centering
\small
\setlength{\fboxsep}{6pt}
\fbox{
\begin{minipage}{0.97\linewidth}
\textbf{Clinical trial example (NLI4CT).}

\vspace{4pt}
\textbf{Source.}
Eligibility criteria from a breast cancer clinical trial: histologically or cytologically confirmed primary invasive breast cancer; primary tumor $>2$ cm; HER2 overexpression/amplification confirmed; no prior breast-cancer therapy; \emph{only Japanese women are eligible}. 

\vspace{4pt}
\textbf{Feasible claim.}
Patients must have histologically or cytologically confirmed HER2-positive invasive breast cancer, with a primary tumor larger than 2 cm in diameter, to participate in the trial.

\vspace{4pt}
\textbf{Adversarial negative.}
To participate in this trial for HER2-positive invasive breast cancer, women with a primary tumor larger than 2 cm in diameter are eligible for enrollment regardless of their geographic or ethnic background.

\vspace{6pt}
\textbf{Reasoning graph.}

\begin{tabular}{@{}p{0.08\linewidth}p{0.87\linewidth}@{}}
\textbf{O} &
O1: Primary invasive breast cancer is required. \\
& O2: Primary tumor must be larger than 2 cm. \\
& O3: HER2 overexpression/amplification is required. \\
& O4: No prior breast-cancer therapy. \\
& O5: \emph{Only Japanese women are eligible.} \\[4pt]

\textbf{C} &
C1: The trial requires invasive breast cancer together with HER2-positive status. \\
& C2: The size threshold and no-prior-therapy criterion define a treatment-naive population. \\
& \textbf{C3 (original):} Eligibility is restricted to Japanese women with primary invasive breast cancer. \\
& \textbf{C3 (corrupted):} Eligibility is open to women with primary invasive breast cancer regardless of nationality. \\[4pt]

\textbf{I} &
I1: The trial targets HER2-positive tumors receiving HER2-directed therapy. \\
& I2: The displayed eligibility constraints characterize who may enroll in the study. \\
\end{tabular}

\vspace{6pt}
\textbf{Edit type.}
\textbf{Population/scope shift.} The corruption preserves the salient disease and tumor-size constraints but removes the nationality restriction encoded in O5 and C3, yielding a globally broader eligibility claim that the source does not support.

\end{minipage}}
\caption{Illustrative hard-negative generation for NLI4CT. A single graph node is corrupted while the surface claim remains locally plausible.}
\label{fig:qual_example_ct}
\end{figure*}

\begin{figure*}[h]
\centering
\small
\setlength{\fboxsep}{6pt}
\fbox{
\begin{minipage}{0.97\linewidth}
\textbf{Scientific table example (SciTab).}

\vspace{4pt}
\textbf{Source.} Table 3 from \textit{Sparse and Structured Visual Attention}.

\vspace{4pt}
\centering
\scriptsize
\renewcommand{\arraystretch}{1.05}
\resizebox{\linewidth}{!}{%
\begin{tabular}{@{}lcccccccccc@{}}
\toprule
 & \textbf{Att. to image} & \textbf{Att. to boxes} & \textbf{TD Y/N} & \textbf{TD Num} & \textbf{TD Other} & \textbf{TD Overall} & \textbf{TS Y/N} & \textbf{TS Num} & \textbf{TS Other} & \textbf{TS Overall} \\
\midrule
softmax       & $\checkmark$ &  & 83.08 & 42.65 & 55.74 & 65.52 & 83.55 & 42.68 & 56.01 & 65.97 \\
sparsemax     & $\checkmark$ &  & 83.08 & 43.19 & 55.79 & 65.60 & 83.33 & 42.99 & 56.06 & 65.94 \\
soft-TVmax    & $\checkmark$ &  & 83.13 & 43.53 & 56.01 & 65.76 & 83.63 & 43.24 & 56.10 & 66.11 \\
sparse-TVmax  & $\checkmark$ &  & 83.10 & 43.30 & 56.14 & 65.79 & 83.66 & 43.18 & 56.21 & 66.17 \\
softmax       &  & $\checkmark$ & 85.14 & 49.59 & 58.72 & 68.57 & 85.56 & 49.54 & 59.11 & 69.04 \\
sparsemax     &  & $\checkmark$ & 85.40 & 50.87 & 58.67 & 68.79 & 85.80 & 50.18 & 59.08 & 69.19 \\
softmax       & $\checkmark$ & $\checkmark$ & 85.33 & 50.49 & 58.88 & 68.82 & 85.58 & 50.42 & 59.18 & 69.17 \\
sparse-TVmax  & $\checkmark$ & $\checkmark$ & 85.35 & 50.52 & 59.15 & 68.96 & 85.72 & 50.66 & 59.22 & 69.28 \\
\bottomrule
\end{tabular}}
\vspace{6pt}
\raggedright

\textbf{Feasible claim.}
Moreover, the model using TVMAX in the final attention layer achieves the highest accuracy, showing that features obtained using the TVMAX transformation are a better complement to bounding box features. 
\vspace{6pt}

\textbf{Adversarial negative.}
The automatic evaluation on VQA-2.0 demonstrates that utilizing the TVmax transformation in the final attention layer consistently yields the highest Test-Standard Overall accuracy across all experimental configurations, regardless of whether the model utilizes attention to images, bounding boxes, or both simultaneously.
\vspace{6pt}

\textbf{Reasoning graph.}
\begin{tabular}{@{}p{0.08\linewidth}p{0.87\linewidth}@{}}
\textbf{O} &
O1: The \textit{sparse-TVmax} row with both image and bounding-box attention has the highest Test-Standard Overall score (69.28). \\
& O2: The \textit{soft-TVmax} row also uses TVmax, but only with image attention, and scores 66.11. \\
& O3: The \textit{softmax} row with both image and bounding-box attention scores 69.17. \\
& O4: The caption states that \textit{sparse-TVmax} and \textit{soft-TVmax} differ in the self-attention mechanism paired with TVmax. \\
& O5: The attention columns specify whether each row uses image attention, bounding-box attention, or both. \\[4pt]

\textbf{C} &
C1: Within the both-attention setting, \textit{sparse-TVmax} slightly outperforms \textit{softmax}. \\
& C2: TVmax rows are not directly interchangeable because they differ in both self-attention pairing and input setting. \\
& C3: The highest score is achieved by one specific TVmax-based configuration, not by every TVmax configuration. \\[4pt]

\textbf{I} &
\textbf{I2 (original):} A configuration using TVmax in the final attention layer achieves the highest accuracy among the listed configurations. \\
& \textbf{I2 (corrupted):} Configurations using TVmax in the final attention layer consistently achieve the highest accuracy across all attention settings. \\
\end{tabular}

\vspace{6pt}
\textbf{Edit type.}
\textbf{Configuration-to-general overreach.} A configuration-specific best result is promoted into a universal claim about all TVmax settings.

\end{minipage}}
\caption{Illustrative hard-negative generation for SciTab. The evidence table is shown in full, followed by the reasoning graph and the single-node corruption that turns a configuration-specific conclusion into an unsupported generalization.}
\label{fig:qual_example_tab}
\end{figure*}

\begin{figure*}[h]
\centering
\small
\setlength{\fboxsep}{6pt}
\fbox{
\begin{minipage}{0.97\linewidth}
\textbf{Scientific chart example (SciVer).}

\vspace{4pt}
\textbf{Source.}

\vspace{4pt}
\begin{center}
\includegraphics[width=0.6\linewidth]{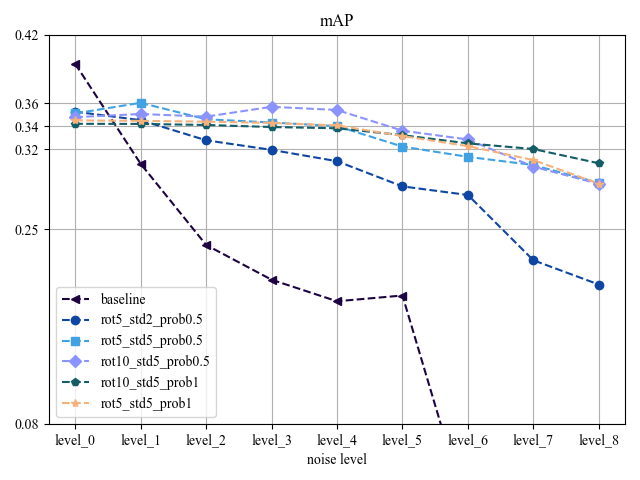}
\end{center}

\vspace{6pt}
\textbf{Feasible claim.}
At noise level\_3, the \texttt{rot10\_std5\_prob0.5} model reaches an mAP of approximately 0.37, which is about 0.15 higher than the baseline's mAP of around 0.22.

\vspace{6pt}
\textbf{Reasoning graph.}

\begin{tabular}{@{}p{0.08\linewidth}p{0.87\linewidth}@{}}
\textbf{O} &
O1: The legend identifies the \texttt{rot10\_std5\_prob0.5} series and the \texttt{baseline} series. \\
& O2: The x-axis shows noise levels \texttt{level\_0} through \texttt{level\_8}, and the y-axis reports mAP. \\
& O3: At \texttt{level\_3}, the \texttt{rot10\_std5\_prob0.5} point is around 0.36--0.37. \\
& O4: At \texttt{level\_3}, the \texttt{baseline} point is around 0.21--0.22. \\
& O5: The vertical separation between the two points at \texttt{level\_3} is visually substantial. \\[4pt]

\textbf{C} &
C1: Using the legend and axis labels, the chart supports a direct local comparison between the two models at \texttt{level\_3}. \\
& C2: At \texttt{level\_3}, \texttt{rot10\_std5\_prob0.5} is clearly above \texttt{baseline} by roughly 0.15 mAP. \\
& C3: This evidence directly supports a pointwise statement about \texttt{level\_3}, but does not by itself establish the same relation across the full noise axis. \\[4pt]

\textbf{I} &
\textbf{I1 (original):} The chart supports a local claim that \texttt{rot10\_std5\_prob0.5} substantially outperforms \texttt{baseline} at \texttt{level\_3}. \\
& \textbf{I1 (corrupted):} The chart supports treating the substantial advantage observed at \texttt{level\_3} as representative of performance across most of the plotted noise spectrum. \\
\end{tabular}

\vspace{6pt}
\textbf{Adversarial negative.}
Across the plotted noise spectrum, the \texttt{rot10\_std5\_prob0.5} model maintains a substantial and consistent mAP advantage over the baseline similar to the gap observed at \texttt{level\_3}, and thus \texttt{rot10\_std5\_prob0.5} outperforms the baseline across most noise levels shown.

\vspace{6pt}
\textbf{Edit type.}
\textbf{Local-to-global overreach.} A valid local comparison at one noise level is promoted into a broader claim about performance across most of the figure, without sufficient visual support.

\end{minipage}}
\caption{Illustrative hard-negative generation for SciVer. A single interpretation-node corruption turns a supported local reading at \texttt{level\_3} into an unsupported global claim about performance across the plotted noise spectrum.}
\label{fig:qual_example_chart}
\end{figure*}

\end{document}